\documentclass{article}

\PassOptionsToPackage{numbers, compress}{natbib}
\usepackage[main, final]{neurips_2026}
\usepackage[utf8]{inputenc} % allow utf-8 input
\usepackage[T1]{fontenc}    % use 8-bit T1 fonts
\usepackage{hyperref}       % hyperlinks
\usepackage{url}            % simple URL typesetting
\usepackage{booktabs}       % professional-quality tables
\usepackage{amsfonts}       % blackboard math symbols
\usepackage{nicefrac}       % compact symbols for 1/2, etc.
\usepackage{microtype}      % microtypography
\usepackage{xcolor}         % colors

\usepackage{multirow}
\usepackage{tabularx}
\usepackage{graphicx}
\usepackage{amsmath}
\usepackage{subcaption}
\usepackage{arydshln}
\usepackage{amsthm}
\newtheorem{theorem}{Theorem}
\usepackage{amssymb}
\usepackage[table]{xcolor}

\usepackage{algorithm}
\usepackage{algorithmic}
\usepackage{booktabs}
\usepackage[most]{tcolorbox}
\usepackage{enumitem}

\tcbuselibrary{listings,breakable,skins}

\newtcblisting{promptbox}[2][]{%
  breakable,
  colback=gray!3,
  colframe=black!60,
  boxrule=0.6pt,
  arc=1mm,
  left=1mm,right=1mm,top=1mm,bottom=1mm,
  title={#2},
  fonttitle=\bfseries,
  listing only,
  listing options={basicstyle=\ttfamily\small,breaklines=true,columns=fullflexible},
  #1
}
\newcommand{\gc}[1]{\cellcolor{gray!14}{#1}}

\title{BALTO: Balanced Token-Level Policy Optimization for Hallucination Mitigation}

\author{
Ning Li\textsuperscript{1}\thanks{Equal contribution.}~~,~~Zixuan Guo\textsuperscript{2}\footnotemark[1]~~,~~Yan Xu\textsuperscript{2},~~Wenbo Fei\textsuperscript{1},~~Yifan Niu\textsuperscript{3},~~Chang Luo\textsuperscript{2}, \\
\textbf{Yasheng Wang\textsuperscript{2}}\thanks{Correspondence to: Yasheng Wang (\texttt{asheryswang@tencent.com}), Weiwen Liu (\texttt{wwliu@sjtu.edu.cn}), and Weinan Zhang (\texttt{wnzhang@sjtu.edu.cn}).}~~,~~\textbf{Weiwen Liu\textsuperscript{1}}\footnotemark[2]~~,~~\textbf{Yong Yu\textsuperscript{1}},~~\textbf{Weinan Zhang\textsuperscript{1}}\footnotemark[2]\\
\textsuperscript{\rm 1}Shanghai Jiao Tong University
\\
  \textsuperscript{\rm 2}Tencent
\\
  \textsuperscript{\rm 3}The Hong Kong University of Science and Technology (Guangzhou)
}

\begin{document}

\maketitle

\begin{abstract}
Hallucinations remain a major obstacle to deploying large language models (LLMs) in knowledge-intensive settings, where generated responses must be faithfully grounded in provided evidence. 
Reinforcement learning (RL) is a promising direction for hallucination mitigation, but response-level faithfulness rewards suffer from a granularity mismatch: localized hallucinations can cause supported content to receive spurious penalties.
Although recent work introduces fine-grained feedback such as claim-level verification and token-level rewards, unbalanced credit assignment can still induce length, verbosity, or optimization-noise biases.
We propose \textbf{BALTO}, a \textbf{Bal}anced \textbf{T}oken-level Policy \textbf{O}ptimization framework for hallucination mitigation.
BALTO extracts checkable factual claims, verifies them against the reference context, and projects claim-level judgments to token-level labels.
A balanced token-level credit assignment mechanism is introduced into the framework. This design redistributes probability mass from unsupported content toward faithful content, rather than suppressing the entire response.
We systematically analyze the limitations of response-level rewards from a theoretical standpoint, and prove BALTO's advantages in training stability and optimization efficiency for hallucination mitigation.
Experiments on ConFiQA, RAGTruth, and FinLLM-Eval show that BALTO achieves the highest faithfulness across all six model--benchmark settings and consistently outperforms existing post-training baselines in Q-Score, demonstrating a stronger faithfulness--informativeness trade-off. The code is released at \url{https://github.com/NingLi670/BALTO}.

\end{abstract}

\section{Introduction}
Large Language Models (LLMs) have demonstrated remarkable capabilities across a wide range of tasks \citep{shao2024deepseekmath, zhang2024codeagent, zhang2026looptool}. However, their real-world deployment remains significantly constrained by faithfulness hallucinations---generating content that lacks factual support or contradicts provided evidence \citep{tamber2025benchmarking}. Such hallucinations pose significant risks in high-stakes domains such as healthcare \citep{kim2025medical}, law \citep{dahl2024large}, finance \citep{kang2023deficiency}, and scientific discovery \citep{xiong2025toward}. Unlike global factual failures, these hallucinations often appear as localized, unsupported spans (e.g., incorrect entities, dates, numerical values, or logical relations) embedded within otherwise coherent and accurate content. This locality makes them particularly difficult to detect and mitigate.

Currently, most reinforcement learning (RL)-based methods for mitigating hallucinations employ response-level rewards, assigning a single scalar score to the entire output. However, such coarse-grained supervision cannot distinguish hallucinated spans from faithful ones: a uniform penalty applied to all tokens may erroneously penalize correct information while providing only weak corrective signals to the actual hallucinated content. This not only diminishes the effectiveness of hallucination mitigation but may also induce reward hacking: models may learn to generate shorter, more conservative responses to avoid penalties, thereby reducing the overall informativeness and utility of the response.

Recent studies~\citep{wen2025policy,li2026reasoning} have moved beyond response-level supervision to incorporate fine-grained feedback signals, such as step-wise faithfulness verification and token-level dense rewards. 
Nevertheless, fine-grained feedback alone does not determine how policy-gradient credit should be assigned. 
Unbalanced token-level rewards may introduce length or faithful-density bias, while token-level updates whose magnitudes remain tied to response-level advantages can still inherit noise from global response-level rewards. 
These observations suggest that hallucination mitigation requires not only fine-grained detection, but also balanced token-level credit assignment.

Based on this premise, we propose \textbf{BALTO} (\textbf{Bal}anced \textbf{T}oken-Level Policy \textbf{O}ptimization), a reinforcement learning framework for fine-grained hallucination mitigation.
BALTO localizes hallucinations at the token level and introduces a balanced credit assignment mechanism: hallucinated tokens receive negative advantages, faithful tokens receive positive compensatory advantages, and neutral tokens are excluded.
By balancing the total negative and positive faithfulness-driven signals within each response, BALTO encourages the model to redistribute probability mass from hallucinated content toward faithful content, rather than suppressing the response as a whole.

We provide both theoretical and empirical analysis to demonstrate the benefits of BALTO.
Theoretically, we show that response-level rewards induce spurious credit assignment under localized hallucinations, whereas balanced token-level credit assignment reduces unnecessary parametric updates and improves optimization stability and efficiency.
Empirically, across ConFiQA, RAGTruth, and FinLLM-Eval with Qwen3-8B and Qwen3-4B, BALTO consistently outperforms both response-level and token-level post-training baselines in both faithfulness and Q-Score across all six model--benchmark settings.
Compared with the strongest post-training baseline in each setting, BALTO improves Q-Score by 2.7 points on average and up to 4.9 points, while improving faithfulness by 3.1 points on average and up to 6.7 points.
These results suggest that BALTO provides a more effective faithfulness--informativeness trade-off than prior post-training approaches, reducing hallucinations without collapsing into overly conservative generation.

Our main contributions are summarized as follows:

\begin{itemize}[leftmargin=*, itemsep=3pt, topsep=0pt]
    \item \textbf{Balanced Token-level Hallucination Mitigation.} We propose BALTO, a token-level RL framework for fine-grained hallucination mitigation. BALTO converts claim-level factuality judgments into token-level optimization signals and uses balanced credit assignment to shift probability mass from hallucinated content toward faithful content.

    \item \textbf{Theoretical Credit Assignment Analysis.} We systematically analyze the credit assignment limitations of response-level faithfulness rewards and show that balanced token-level advantages reduce unnecessary parametric updates, improving training stability and optimization efficiency.

    \item \textbf{SOTA Performance.} Across ConFiQA, RAGTruth, and FinLLM-Eval with Qwen3-8B and Qwen3-4B, BALTO achieves the highest faithfulness in all six model--benchmark settings and consistently outperforms existing post-training methods in Q-Score, demonstrating a stronger faithfulness--informativeness trade-off.
\end{itemize}

\section{Background and Motivation}
\label{sec:background}

\subsection{Pilot Study: Hallucinations Are Token-Sparse but Response-Prevalent}
\label{sec:pilot}

\begin{figure}[t]
\vspace{-5pt}
  \centering
  \begin{minipage}[t]{0.49\textwidth}
    \centering
    \includegraphics[width=1\linewidth]{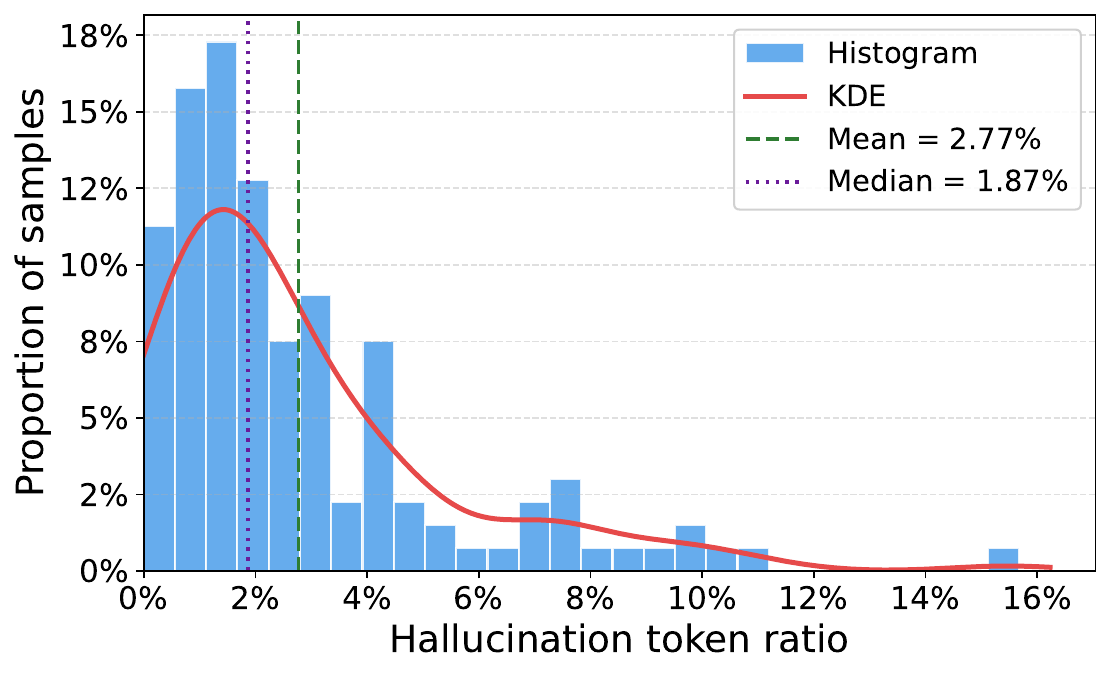}
    \vspace{-1mm}
    \centerline{\small (a) Hallucinated-token ratio per response.}
  \end{minipage}
  \hfill
  \begin{minipage}[t]{0.49\textwidth}
    \centering
    \includegraphics[width=1.0\linewidth]{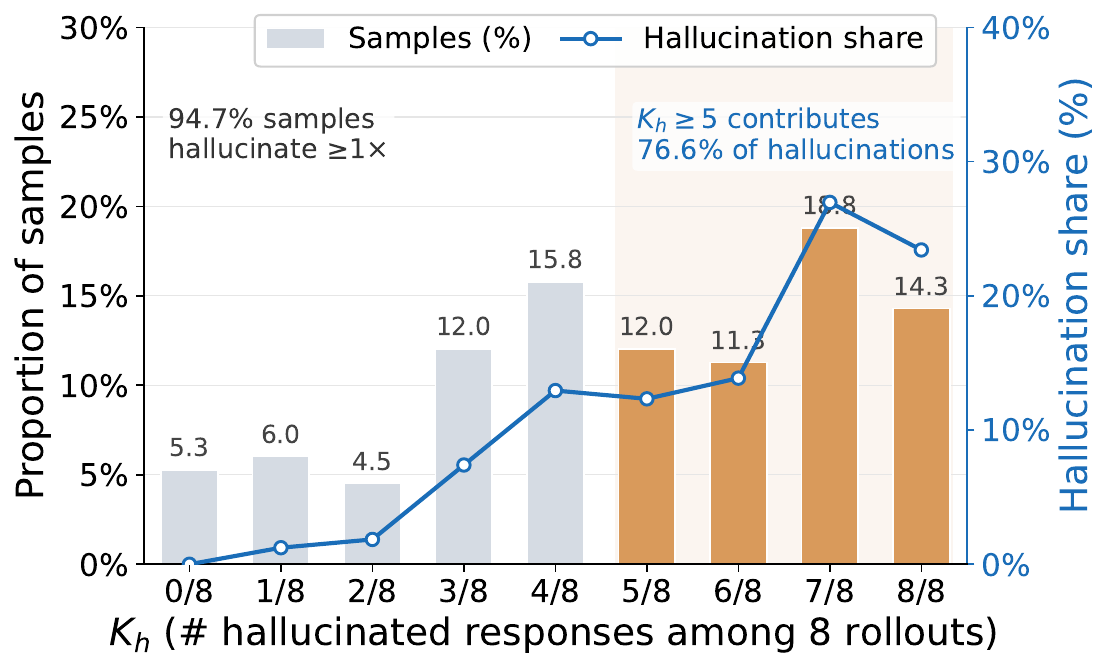}
    \vspace{-1mm}
    \centerline{\small (b) Hallucinated responses among eight rollouts.}
  \end{minipage}
  \caption{
  Pilot study on FinLLM-Eval with Qwen3-8B.
  Hallucinations are sparse within individual responses but prevalent across rollouts.
  The shaded region in (b) marks majority-hallucinating groups.
  }
  \label{fig:pilot_hallucination_distribution}
    \vspace{-10pt}
  
\end{figure}

Faithfulness hallucinations in RAG-style generation are often localized factual errors rather than complete response-level failures.
Prior studies similarly analyze factual errors at the level of claims, atomic facts, or evidence-grounded statements~\citep{ji2023survey, alansari2026large, chen2025learning, ren2025knowrl, yin2026mitigating}.
To quantify this locality, we conduct a pilot study on FinLLM-Eval~\citep{finLLM-Eval}, a financial QA benchmark requiring answers to be grounded in numerical evidence.

As shown in Figure~\ref{fig:pilot_hallucination_distribution}(a), hallucinated tokens account for only a small fraction of each response, with an average ratio of 2.77\% and a median ratio of 1.87\%.
However, this token-level sparsity does not imply response-level sparsity.
Figure~\ref{fig:pilot_hallucination_distribution}(b) shows that 94.7\% of rollout groups contain at least one hallucinated response, and groups with majority-hallucinating rollouts ($K_h\geq5$) contribute 76.6\% of all hallucinations.
Thus, hallucinations are sparse within responses but prevalent across sampled responses.
This dual pattern makes response-level RL supervision inefficient: partially faithful responses with only a few erroneous tokens may receive the same binary penalty as severely hallucinated responses, motivating fine-grained and balanced token-level credit assignment.

\subsection{Response-Level RL for Hallucination Mitigation}
\label{sec:response_level_rl}

RL has become a common post-training paradigm for improving LLM faithfulness and alignment~\citep{shao2024deepseekmath, song2025r1, chen2503research, wen2025policy, li2026reasoning}.
In a RAG setting, given an input $x=(q,d)$ with query $q$ and reference documents $d$, the model generates a response $y=(y_1,\ldots,y_T)$ according to an autoregressive policy.
In GRPO-style optimization, for each input $x$, a group of $K$ responses $\{y_i\}_{i=1}^{K}$ is sampled from the old policy $\pi_{\theta_{\mathrm{old}}}$.
Given a response-level reward $r_i=R(x,y_i)$, GRPO computes the group-relative advantage:
$A_i^{\mathrm{resp}}=\frac{r_i-\operatorname{mean}\left(\boldsymbol{r}\right)}{\operatorname{std}\left(\boldsymbol{r}\right)}$, where $\boldsymbol{r}=(r_1,\ldots,r_K)$.
The advantage formula applies when $\operatorname{std}(\boldsymbol{r})>0$. This scalar advantage is broadcast to all tokens in the response. With lower and upper clipping widths $\epsilon_{\mathrm{low}},\epsilon_{\mathrm{high}}\geq0$, the clipped GRPO objective is
{\small
\begin{align}
&\mathcal{J}_{\mathrm{GRPO}}(\theta)
=
\mathbb{E}_{\{y_i\}_{i=1}^K \sim \pi_{\theta_{\mathrm{old}}}(\cdot\mid x)} \notag \\
&\left[
\frac{1}{K}
\sum_{i=1}^{K}\frac{1}{T_i}
\sum_{t=1}^{T_i}
\min
\left(
\frac{\pi_\theta(y_{i,t} \mid x, y_{i,<t})}{\pi_{\theta_{\mathrm{old}}}(y_{i,t} \mid x, y_{i,<t})} A_i^{\mathrm{resp}},
\operatorname{clip}\left(\frac{\pi_\theta(y_{i,t} \mid x, y_{i,<t})}{\pi_{\theta_{\mathrm{old}}}(y_{i,t} \mid x, y_{i,<t})}, 1-\epsilon_{\mathrm{low}}, 1+\epsilon_{\mathrm{high}}\right) A_i^{\mathrm{resp}}
\right)
\right]
.
\end{align}
}

This response-level assignment is problematic for localized hallucinations.
Let $\mathcal{H}_i$ and $\mathcal{F}_i$ denote hallucinated and faithful factual tokens in response $y_i$.
If a few hallucinated tokens cause a low faithfulness reward, then $A_i^{\mathrm{resp}}<0$, and the factual-token component of the update is proportional to
\begin{equation}
A_i^{\mathrm{resp}}
\left(
\sum_{t\in\mathcal{H}_i}
\nabla_\theta \log \pi_\theta(y_{i,t}\mid x,y_{i,<t})
+
\sum_{t\in\mathcal{F}_i}
\nabla_\theta \log \pi_\theta(y_{i,t}\mid x,y_{i,<t})
\right).
\label{eq:spurious_credit}
\end{equation}
The first term suppresses hallucinated tokens, but the second term also suppresses faithful factual tokens.
Since hallucinated tokens are typically sparse, this spurious penalty can dominate the faithfulness-driven update, reducing informativeness or encouraging overly conservative responses.
This motivates moving beyond response-level rewards toward fine-grained hallucination feedback.

\subsection{Fine-Grained Hallucination Feedback and the Remaining Assignment Gap}
\label{sec:fine_grained_gap}

A natural remedy to response-level credit mismatch is to introduce finer-grained hallucination feedback.
Prior work has explored claim-level correctness, atomic fact verification, statement-level faithfulness assessment, and dense process supervision for faithfulness alignment~\citep{lightman2023let, chen2025learning, dou2026baichuan, ren2025knowrl, yin2026mitigating}.
Recent RL-based hallucination mitigation methods further move toward token-level optimization.
FSPO adjusts token-level advantages using step-wise faithfulness verification~\citep{li2026reasoning}, while RLFH decomposes responses into atomic facts and converts statement-level truthfulness and informativeness feedback into token-level dense rewards~\citep{wen2025policy}.
We provide a broader discussion of fine-grained hallucination mitigation and process supervision in Appendix~\ref{sec:related}.

However, fine-grained feedback alone does not determine how policy-gradient credit should be assigned.
Penalty-only token-level feedback may encourage shorter or less informative responses, while assigning fixed positive rewards to all faithful factual tokens may bias the model toward verbose factual padding.
Similarly, token-level updates whose magnitudes remain tied to response-level advantages can still inherit noise from global response-level rewards.

These limitations motivate a balanced token-level assignment principle: hallucinated factual tokens should receive localized negative feedback, faithful factual tokens should receive compensatory positive feedback, and the total faithfulness-driven signal should be balanced within each response.
In the next section, we instantiate this principle as a balanced token-level policy optimization framework for hallucination mitigation.

\section{Methodology}

\begin{figure}
\vspace{-10pt}
    \centering
    \includegraphics[width=1\linewidth]{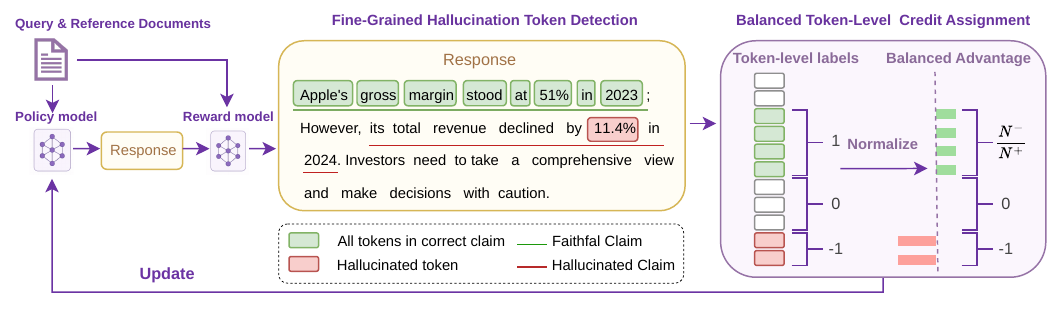}
    \caption{Overview of the BALTO framework is demonstrated. Initially, the reward model localize hallucinations at the token level. Subsequently, the balanced credit mechanism assigns token-level advantages: hallucinated tokens (red) receive -1 advantage , faithful factual tokens (green) receive positive compensatory advantages of $N^-/N^+$ , where $N^-$ and $N^+$ are the numbers of hallucinated and faithful tokens respectively, and neutral tokens receive 0 advantage (white). }
    \label{fig:balto}
    \vspace{-15pt}
\end{figure}

We propose \textbf{BALTO}: \textbf{Bal}anced \textbf{T}oken-level Policy \textbf{O}ptimization, a fine-grained reinforcement learning framework for hallucination mitigation.
As illustrated in Figure~\ref{fig:balto}, BALTO consists of two stages.
First, an LLM-based faithfulness evaluator extracts and verifies checkable factual claims from each sampled response.
Second, BALTO projects claim-level judgments to token-level labels, constructs balanced token-level credits, and optimizes the policy with a GRPO-style clipped objective.

\subsection{Fine-Grained Hallucination Token Detection}
\label{sec:detection}

Given an input $x=(q,d)$ and a sampled response
$y_i=(y_{i,1},\ldots,y_{i,T_i})$, BALTO first identifies checkable factual claims in the response.
A claim is defined as a minimal text span that expresses a verifiable factual assertion and can be independently checked against the reference document.
We represent the extracted claims as $\mathcal{C}_i=\{c_{i,m}\}_{m=1}^{M_i}$, where each $c_{i,m}\subseteq\{1,\ldots,T_i\}$ denotes the token indices covered by the $m$-th claim.
In practice, we focus on claims involving named entities, numerical values, dates, temporal expressions, attributes, and relations, which are common sources of faithfulness errors and align with atomic faithfulness evaluation~\citep{pagnoni2021understanding,nan2021entity,min2023factscore}.
The faithfulness evaluator verifies each claim against the reference document $d$ and partitions the claims into faithful and hallucinated claim sets: 
\begin{align*}
\mathcal{C}_i^{+}
&=
\{c_{i,m}\in\mathcal{C}_i:
c_{i,m}\text{ is supported by }d\}, \\
\mathcal{C}_i^{-}
&=
\{c_{i,m}\in\mathcal{C}_i:
c_{i,m}\text{ is unsupported or contradicted by }d\}.
\end{align*}

We then project claim-level verification results back to token-level labels using the text-to-token alignment procedure in Appendix~\ref{app:token_alignment}.
For faithful claims, the projection is straightforward: all tokens covered by the claim are labeled as faithful factual tokens.
For hallucinated claims, however, assigning negative feedback to the entire claim span can be overly coarse, since a claim may contain only a small erroneous part.
Therefore, for each hallucinated claim $c_{i,m}\in\mathcal{C}_i^{-}$, BALTO further localizes a minimal erroneous token subset $e_{i,m}\subseteq c_{i,m}$, where $e_{i,m}$ contains the tokens directly responsible for the factual error, such as an incorrect entity, number, date, attribute, or relation.
If the erroneous part cannot be reliably isolated, we conservatively set $e_{i,m}=c_{i,m}$. Formally, the token-level label is defined as
\begin{equation}
z_{i,t}
=
\begin{cases}
-1,
& \exists c_{i,m}\in\mathcal{C}_i^{-}
\text{ such that } t\in e_{i,m},\\[3pt]
+1,
& \exists c_{i,m}\in\mathcal{C}_i^{+}
\text{ such that } t\in c_{i,m}
\text{ and } t\notin \bigcup_{c_{i,m'}\in\mathcal{C}_i^{-}} e_{i,m'},\\[3pt]
0,
& \text{otherwise}.
\end{cases}
\label{eq:token_label_projection}
\end{equation}
Here, $z_{i,t}=-1$ indicates that token $y_{i,t}$ is directly responsible for a hallucinated factual claim, $z_{i,t}=+1$ indicates that it belongs to a faithful factual claim, and $z_{i,t}=0$ indicates that it is either outside checkable factual claims or is a faithful token inside a hallucinated claim.
When a token is included in both faithful and hallucinated projections due to overlapping claims, the hallucination label takes precedence.

Finally, we define the hallucinated factual token set, faithful factual token set, and neutral token set as
\begin{equation}
\mathcal{H}_i=\{t\mid z_{i,t}=-1\},
\qquad
\mathcal{F}_i=\{t\mid z_{i,t}=+1\},
\qquad
\mathcal{N}_i=\{t\mid z_{i,t}=0\}.
\label{eq:token_sets}
\end{equation}

This projection differs from statement-level reward mapping, which typically assigns feedback to a whole matched statement or sentence~\citep{wen2025policy}.
BALTO instead localizes the factual error within a hallucinated claim and applies negative feedback only to the minimal erroneous tokens.
For the hallucinated claim in Figure~\ref{fig:balto} that ``total revenue declined by 11.4\% in 2024'', BALTO penalizes the token span corresponding to ``11.4\%'' rather than the entire claim.
This provides precise negative supervision while avoiding unnecessary suppression of surrounding faithful context.
The resulting token sets $\mathcal{H}_i$, $\mathcal{F}_i$, and $\mathcal{N}_i$ are then used to construct balanced token-level advantages.

\subsection{Balanced Token-Level Credit Assignment}

BALTO directly assigns token-level advantages based on the factual status of tokens.
For response $y_i$, let $N_i^{-}=|\mathcal{H}_i|$, $N_i^{+}=|\mathcal{F}_i|$ denote the number of hallucinated factual tokens and faithful factual tokens, respectively.
We define the balanced token-level advantage as
\begin{equation}
A_{i,t}^{\mathrm{bal}}
=
\begin{cases}
-1,
& t\in \mathcal{H}_i, \\[4pt]
\dfrac{N_i^{-}}{N_i^{+}},
& t\in \mathcal{F}_i,\ N_i^{-}>0,\ N_i^{+}>0, \\[10pt]
0,
& t\in \mathcal{N}_i \text{ or } N_i^{-}=0.
\end{cases}
\label{eq:balanced_advantage}
\end{equation}
Thus, hallucinated factual tokens receive a negative advantage, faithful factual tokens receive a positive compensation advantage, and tokens outside checkable factual claims do not contribute to faithfulness-driven optimization~\footnote{For faithful responses, where $N_i^{-}=0$, all faithfulness-driven advantages are set to zero, avoiding an incentive to generate unnecessarily long responses containing redundant safe claims.
If $N_i^{-}>0$ but $N_i^{+}=0$, BALTO applies only the negative correction, since no faithful factual token is available for within-response compensation.}.
When both hallucinated and faithful factual tokens exist, BALTO satisfies an intra-response balance property:
\begin{equation}
\sum_{t=1}^{T_i} A_{i,t}^{\mathrm{bal}}
=
\sum_{t\in\mathcal{H}_i}(-1)
+
\sum_{t\in\mathcal{F}_i}
\frac{N_i^{-}}{N_i^{+}}
=
-N_i^{-}
+
N_i^{+}\frac{N_i^{-}}{N_i^{+}}
=
0.
\label{eq:balance_property}
\end{equation}
Therefore, the total negative advantage assigned to hallucinated factual tokens is exactly matched by the total positive advantage assigned to faithful factual tokens.
This encourages the model to redistribute probability mass from hallucinated factual content to supported factual content, rather than suppressing the entire response.

\textbf{BALTO Objective} Given an input $x$, we sample a group of $K$ responses $\{y_i\}_{i=1}^{K}$ from the old policy $\pi_{\theta_{\mathrm{old}}}$.
For each response, we obtain token sets $\mathcal{H}_i$, $\mathcal{F}_i$, and $\mathcal{N}_i$ using the faithfulness evaluator, and assign balanced advantages according to Eq.~\ref{eq:balanced_advantage}.
The BALTO objective is
{\small
\begin{align}
&\mathcal{J}_{\mathrm{BALTO}}(\theta)
=
\mathbb{E}_{\{y_i\}_{i=1}^K \sim \pi_{\theta_{\mathrm{old}}}(\cdot\mid x)} \notag \\
&\left[
\frac{1}{K}
\sum_{i=1}^{K}\frac{1}{Z_i}
\sum_{t\in\mathcal{H}_i\cup\mathcal{F}_i}
\min
\left(
\frac{\pi_\theta(y_{i,t} \mid x, y_{i,<t})}{\pi_{\theta_{\mathrm{old}}}(y_{i,t} \mid x, y_{i,<t})} A_{i,t}^{\mathrm{bal}},
\operatorname{clip}\left(\frac{\pi_\theta(y_{i,t} \mid x, y_{i,<t})}{\pi_{\theta_{\mathrm{old}}}(y_{i,t} \mid x, y_{i,<t})}, 1-\epsilon_{\mathrm{low}}, 1+\epsilon_{\mathrm{high}}\right) A_{i,t}^{\mathrm{bal}}
\right)
\right]
.
\label{eq:balto_objective}
\end{align}
}
where
\begin{equation}
Z_i=\max\left(1,|\mathcal{H}_i|+|\mathcal{F}_i|\right)
\end{equation}
normalizes the update over factual tokens.
Since tokens in $\mathcal{N}_i$ have zero advantage, they do not contribute to the faithfulness-driven update.

Compared with response-level GRPO, BALTO provides localized and balanced supervision even when hallucinations are widespread at the response level.
Compared with unbalanced token-level rewards, BALTO explicitly controls the total positive and negative advantage mass within each response.
As a result, the policy receives precise corrective signals for hallucinated content while preserving response informativeness.

\section{Theoretical Analysis}
\label{sec:theory}

In this section, we analyze the theoretical advantages of BALTO over standard response-level GRPO. We first examine gradient variance under a correlated-token model, then discuss optimization efficiency across training stages.

To evaluate the optimization dynamics, we model the response-level reward $r$ in GRPO as a Bernoulli variable: $r \sim \mathcal{B}(p)$, where $0<p<1$ is the probability for generating a faithful response. Here, $r=1$ denotes a faithful response ($y^+$) and $r=0$ denotes a hallucinated response ($y^-$). Therefore the mean and std of reward $r$ are $p$ and $\sqrt{p(1-p)}$. For this analysis, we use the population-normalized counterpart of the group-relative advantage in Section~\ref{sec:response_level_rl}: $A = (r - p)/\sqrt{p(1-p)}$. In contrast, BALTO applies balanced local advantages $A_{t}^{\mathrm{bal}}$ as defined in Eq.~\ref{eq:balanced_advantage}, satisfying $\sum_{t=1}^T A_t^{\mathrm{bal}}=0$ when $N^-=0$ or $N^+>0$ (Eq.~\ref{eq:balance_property}).

Let a generated response $\mathbf{y}$ of length $T$ be partitioned into three disjoint sets as defined in Eq.~\ref{eq:token_sets}: the hallucinated factual token set $\mathcal{H}$ (size $N^-$), the faithful factual token set $\mathcal{F}$ (size $N^+$), and the neutral token set $\mathcal{N}$. We define the per-token gradient operator as $g_t = \nabla_\theta \log \pi_\theta(y_t \mid x,\mathbf{y}_{<t})$.

% ---------------------------------------------------------
\subsection{Enhancing Training Stability via Variance Reduction}
\label{subsec:stability_variance}

To compare GRPO's response-wide credit with BALTO's localized updates, we analyze the raw gradient sums $\hat{\nabla}J_{\mathrm{resp}}=A\sum_t g_t$ and $\hat{\nabla}J_{\mathrm{BALTO}}=\sum_t A_t^{\mathrm{bal}}g_t$ before clipping and length normalization. For fixed length $T$, condition on the reward and token-label configuration $\mathcal{L}$ and assume: \textbf{(A1)} $\operatorname{Cov}(g_i,g_j\mid\mathcal{L})=\rho_{ij}\Sigma$, where $\Sigma$ is a common nonzero positive-semidefinite covariance matrix and $P=(\rho_{ij})$ is a correlation matrix with $\rho_{ii}=1$; and \textbf{(A2)} the mean pairwise correlation $\bar\rho=\frac{1}{T(T-1)}\sum_{i\neq j}\rho_{ij}\geq0$ for $T>1$. These simplifying assumptions allow nonuniform correlations of either sign while requiring their average to be nonnegative.

\begin{theorem}[Variance Reduction with Correlated Token Gradients]
\label{thm:stability}
Under (A1)--(A2), define the scalar variance $V_m=\mathbb{E}_{\mathcal{L}}[\operatorname{tr}\operatorname{Cov}(\hat{\nabla}J_m\mid\mathcal{L})]$ for $m\in\{\mathrm{resp},\mathrm{BALTO}\}$, and let $\sigma_g^2=\operatorname{tr}(\Sigma)$. Then
$V_{\mathrm{resp}}\geq T\sigma_g^2$ and
$V_{\mathrm{BALTO}}\leq4\mathbb{E}[(N^-)^2]\sigma_g^2$.
Consequently, $4\mathbb{E}[(N^-)^2]<T$ implies $V_{\mathrm{BALTO}}<V_{\mathrm{resp}}$, even when token gradients are correlated.
\end{theorem}

\begin{proof}[Proof Sketch]
For token weights $w_t$, (A1) gives conditional scalar variance $S\sigma_g^2$, where
$S=\sum_t w_t^2+\sum_{i\neq j}w_iw_j\rho_{ij}$.
Since $|\rho_{ij}|\leq1$, BALTO's weights $-1$, $N^-/N^+$, and $0$ yield
$S_{\mathrm{BALTO}}\leq(\sum_t|w_t|)^2=4(N^-)^2$ when $N^-,N^+>0$. The same bound holds in the boundary cases.
For GRPO, $S_{\mathrm{resp}}=A^2[T+T(T-1)\bar\rho]\geq A^2T$ by (A2).
Taking expectations and using $\mathbb{E}[A^2]=1$ proves the comparison. Thus the correlated-case bound scales as $\mathcal{O}(\mathbb{E}[(N^-)^2]\sigma_g^2)$, replacing the original independence-based bound. Appendix~\ref{proof1} provides the complete derivation and relates these raw-sum bounds to the implemented objective.
\end{proof}

% ---------------------------------------------------------
\subsection{Maximizing Optimization Efficiency across Training Stages}
\label{subsec:efficiency_asymptotic}

As training progresses, the probability of generating a faithful response, denoted by $p$, gradually increases. An efficient optimization method should provide strong and reliable gradients throughout this process. However, standard global normalization suffers from gradient starvation during early training (low $p$) and disproportionate penalties during late-stage convergence (high $p$). BALTO resolves both extremes by using bounded token-level advantages, guaranteeing efficient optimization across all stages.

\begin{theorem}[High Optimization Efficiency in Asymptotic Regimes]
\label{thm:efficiency}
As $p \to 0$, response-level GRPO suffers from vanishing expected gradients ($\lVert\mathbb{E}[\hat{\nabla} J_{\mathrm{resp}}]\rVert_2 \to 0$), severely limiting early-stage learning efficiency. In contrast, BALTO maintains strong local learning signals to ensure efficient exploration. As $p \to 1$, response-level GRPO triggers unbounded divergent penalties ($A^- \to -\infty$), destroying late-stage optimization efficiency. BALTO, however, maintains strictly bounded updates ($A_t^{\mathrm{bal}} \in[-1, 1]$), ensuring efficient and stable final convergence.
\end{theorem}

\begin{proof}[Proof Sketch]
For response-Level GRPO, the expected gradient magnitude is bounded by the normalization coefficient $\sqrt{p(1-p)}$. $\lim_{p \to 0} \lVert\mathbb{E}[\hat{\nabla} J_{\mathrm{resp}}]\rVert_2 \propto \lim_{p \to 0} \sqrt{p} = 0$, which leads to gradient starvation when faithful responses are rare. As $p \to 1$, the penalty $A^- = -\sqrt{p/(1-p)} \to -\infty$, causing divergent updates. In contrast, BALTO's expected gradient magnitude is proportional to $(1-p)$. As $p \to 0$, $(1-p) \to 1$, so BALTO still provides dense and strong correction signals. As $p \to 1$, $(1-p) \to 0$, naturally decaying the gradient. Meanwhile, local advantages $A_t^{\mathrm{bal}}$ remain strictly bounded in $[-1, 1]$ preventing instability. (Detailed derivation in Appendix \ref{proof2}).
\end{proof}

\section{Experiments}
\subsection{Experimental Settings}

\textbf{Datasets.}
We evaluate our method on three benchmarks that cover various hallucination scenarios. 
\textbf{ConFiQA} \citep{bi2025context} is a counterfactual QA dataset constructed from conflicting evidence, and we focus on the most challenging Multi-Conflicts (MC) subset. 
\textbf{RAGTruth} \citep{niu2024ragtruth} is a RAG hallucination corpus with diverse data sources and task formats. 
\textbf{FinLLM-Eval} \citep{finLLM-Eval} is a financial domain QA dataset containing large amounts of financial facts and numerical information. We evaluate our method on FinLLM-Eval, a high-stakes benchmark that demands precise answers, as the financial domain is particularly sensitive to hallucinations: even minor factual or numerical errors can lead to misleading analyses and consequential decision-making risks.

\textbf{Baselines.}
We use Qwen3-4B and Qwen3-8B \citep{yang2025qwen3} as base models and compare with three categories of methods discussed in Appendix~\ref{sec:related}. 
(1) \textbf{Data-driven methods}: we construct preference pairs and train models using SFT and DPO. 
% Due to limited performance, we report results only on Qwen3-8B. 
(2) \textbf{Response-level RL}: we implement GRPO \citep{shao2024deepseekmath} with two reward designs, including a binary hallucination reward (GRPO$_B$) and a dense reward based on the proportion of correct claims (GRPO$_D$). Most existing response-level hallucination optimization methods can be viewed as variants of GRPO$_D$. 
(3) \textbf{Token-level RL}: we compare with FSPO \citep{li2026reasoning}, which performs token-level advantage reshaping on the response level reward. We do not directly compare with RLFH \citep{wen2025policy}, a PPO-based token-level method, due to its discrepancy on the self-assessment setting. Nevertheless, we perform RLFH-style experiments with its core components and conduct an ablation study in Section~\ref{sec:ablation}.

\textbf{Evaluation.}
Following prior work~\citep{min2023factscore, yin2026mitigating}, we use an LLM-based evaluation pipeline to extract claims from model responses and verify their faithfulness against reference evidence. 
We report \textbf{faithfulness} (Faith.), which measures the proportion of responses without hallucinations, and \textbf{informativeness} (Info.), which evaluates whether the response preserves sufficient information compared to a reference response generated by the base model. We define \textbf{Q-Score} = Faith. $\times$ Info. to assert the comprehensive level of faithfulness and informativeness. To further evaluate the response helpfulness, for ConFiQA, we additionally report answer \textbf{accuracy} (Acc.). For RAGTruth and FinLLM-Eval, following~\citep{li2026knowledgelevel, chen2025learning}, we conduct pairwise comparisons between RL baselines and our BALTO by reporting the \textbf{win rate} (WR). 
Win Rate uses a separate judge, DeepSeek-v3.2, while ConFiQA Acc. uses rule-based reference-answer matching. Both provide complementary assessments independent of the training verifier, Qwen3.5-35B-A3B.
More evaluation details are provided in Appendix~\ref{app:evaluation}.

\textbf{Implementation Details.}
We implement our method and the RL baselines using the verl \citep{sheng2024hybridflow} framework. 
Models are trained on the training split of each dataset and evaluated on the test set. 
When training, we use a batch size of 256 and a mini-batch size of 64, with 8 rollouts per prompt. 
The learning rate is set to $1 \times 10^{-6}$. 
Following DAPO \citep{yu2025dapo}, we use asymmetric clipping with $\epsilon_{\mathrm{low}}=0.2$ and $\epsilon_{\mathrm{high}}=0.28$. More details are in Appendix \ref{app:implementation-details}.

\subsection{Main Results}

\begin{table*}[htpb]
\centering
\caption{Main results (\%) on ConFiQA, RAGTruth, and FinLLM-Eval using Qwen3-8B and Qwen3-4B. \textbf{Bold} and \underline{underlined} scores represent the best and second-best results. \textit{[ref.]} denotes the reference method for relative metrics, with Info. computed against Base and WR computed against BALTO.}
\label{tab:main_results}
\resizebox{\textwidth}{!}{
\begin{tabular}{llcccccccccccc}
\toprule
\multirow{2}{*}{Model} 
& \multirow{2}{*}{Method}
& \multicolumn{4}{c}{ConFiQA} 
& \multicolumn{4}{c}{RAGTruth} 
& \multicolumn{4}{c}{FinLLM-Eval} \\
\cmidrule(lr){3-6} \cmidrule(lr){7-10} \cmidrule(lr){11-14}
& & Faith. & Info. & Q-Score & Acc.
  & Faith. & Info. & Q-Score & WR
  & Faith. & Info. & Q-Score & WR \\
\midrule
\multirow{7}{*}{Qwen3-8B}
& Base & 68.2 & \textit{[ref.]} & 68.2 & 70.6 & 40.4 & \textit{[ref.]} & 40.4 & - & 34.6 & \textit{[ref.]} & 34.6 & - \\
& \gc{SFT}  & \gc{71.2} & \gc{68.0} & \gc{48.4} & \gc{71.6} & \gc{56.2} & \gc{66.7} & \gc{37.5} & \gc{-} & \gc{48.1} & \gc{\underline{70.7}} & \gc{34.0} & \gc{-} \\
& \gc{DPO}  & \gc{72.8} & \gc{84.0} & \gc{61.2} & \gc{62.0} & \gc{52.4} & \gc{11.6} & \gc{6.1} & \gc{-} & \gc{41.4} & \gc{60.2} & \gc{24.9} & \gc{-} \\
& GRPO$_{B}$ & 92.8 & 93.5 & 86.8 & 73.8 & 89.8 & 55.6 & 49.9 & 45.1 & 90.2 & 39.1 & 35.3 & 44.0 \\
& GRPO$_{D}$  & 94.5 & 95.2 & 90.0 & \underline{74.4} & \underline{92.0} & \underline{73.8} & \underline{67.9} & 37.2 & \underline{91.7} & 51.1 & 46.9 & 37.2 \\
& \gc{FSPO} & \gc{\underline{94.8}} & \gc{\textbf{97.3}} & \gc{\underline{92.2}} & \gc{66.4} & \gc{73.1} & \gc{\textbf{84.9}} & \gc{62.1} & \gc{34.2} & \gc{59.4} & \gc{\textbf{82.0}} & \gc{\underline{48.7}} & \gc{46.6} \\
& BALTO & \textbf{97.5} & \underline{96.2} & \textbf{93.8} & \textbf{80.2} & \textbf{98.4} & 71.1 & \textbf{70.0} & \textit{[ref.]} & \textbf{92.5} & 57.9 & \textbf{53.6} & \textit{[ref.]} \\
\midrule
\multirow{5}{*}{Qwen3-4B}
& \gc{Base} & \gc{80.3} & \gc{\textit{[ref.]}} & \gc{80.3} & \gc{36.8} & \gc{62.2} & \gc{\textit{[ref.]}} & \gc{\textbf{62.2}} & \gc{-} & \gc{57.1} & \gc{\textit{[ref.]}} & \gc{\textbf{57.1}} & \gc{-} \\
& GRPO$_{B}$ & \underline{97.8} & \underline{96.3} & 94.2 & 36.2 & 89.8 & 21.3 & 19.1 & 39.8 & \underline{82.0} & 23.3 & 19.1 & 38.8 \\
& GRPO$_{D}$  & 88.0 & 92.8 & 81.7 & \underline{37.8} & \underline{93.1} & 50.9 & 47.4 & 35.2 & 78.2 & \textbf{43.6} & 34.1 & 57.7 \\
& \gc{FSPO} & \gc{95.8} & \gc{\textbf{98.5}} & \gc{\underline{94.4}} & \gc{36.3} & \gc{87.3} & \gc{\textbf{57.3}} & \gc{50.0} & \gc{22.7} & \gc{77.4} & \gc{\underline{42.9}} & \gc{33.2} & \gc{42.5} \\
& BALTO & \textbf{99.0} & 96.2 & \textbf{95.2} & \textbf{38.6} & \textbf{94.0} & \underline{56.7} & \underline{53.3} & \textit{[ref.]} & \textbf{88.7} & 42.1 & \underline{37.3} & \textit{[ref.]} \\
\bottomrule
\end{tabular}
}
\end{table*}

Table~\ref{tab:main_results} reports the main results on ConFiQA, RAGTruth, and FinLLM-Eval. 
Overall, BALTO achieves the highest faithfulness across all datasets and both model scales, and obtains the best Q-Score among all post-training methods. 
Compared with the strongest post-training baseline in each setting, BALTO improves faithfulness by 3.1 points on average and up to 6.7 points, while improving Q-Score by 2.7 points on average and up to 4.9 points. 
These gains are particularly pronounced over response-level RL methods (GRPO$_B$ and GRPO$_D$), suggesting that balanced token-level credit assignment provides more precise optimization signals for localized hallucinations.

Although FSPO achieves higher informativeness on some datasets, it often does so at the cost of faithfulness, leading to a weaker overall faithfulness--informativeness trade-off. 
In contrast, BALTO consistently improves faithfulness while maintaining competitive informativeness, indicating that it suppresses unsupported content without simply shortening or over-conservatizing responses. 
Moreover, BALTO attains the highest accuracy on ConFiQA for both model scales, improving over the strongest post-training baseline by up to 5.8 points, and serves as the reference in most pairwise response comparisons on RAGTruth and FinLLM-Eval. 
These results show that BALTO preserves overall response quality while achieving stronger hallucination mitigation.

Additional experiments with Llama-3.1-8B~\citep{grattafiori2024llama} on RAGTruth and FinLLM-Eval also show the highest faithfulness and Q-Score among the evaluated methods, supporting generalization beyond the Qwen model family (Appendix~\ref{app:llama_results}).

\subsection{Ablation Study}
\label{sec:ablation}
\begin{table}[htbp]
    % \vspace{-15pt}
\centering
\caption{Ablation study on positive advantage assignment. ``PA = 0'' removes the positive advantage for faithful tokens, and ``PA = 0.3'' / ``PA = 1'' assign fixed positive advantage values.}
\label{tab:ablation_pos_reward_final}
\small{
\begin{tabular}{lccccccccc}
\toprule
\multirow{2}{*}{Method} 
& \multicolumn{3}{c}{ConFiQA} 
& \multicolumn{3}{c}{RAGTruth} 
& \multicolumn{3}{c}{FinLLM-Eval} \\
\cmidrule(lr){2-4} \cmidrule(lr){5-7} \cmidrule(lr){8-10}
& Faith. & Info. & Q-Score 
& Faith. & Info. & Q-Score 
& Faith. & Info. & Q-Score \\
\midrule
BALTO 
& \textbf{97.5} & \textbf{96.2} & \textbf{93.8} 
& \textbf{98.4} & 71.1 & \textbf{70.0} 
& \textbf{92.5} & 57.9 & \textbf{53.6} \\
PA = 0
& 91.3 & 92.2 & 84.2 
& 86.0 & 64.0 & 55.0 
& 78.2 & 36.1 & 28.2 \\
PA = 0.3
& 84.0 & 93.2 & 78.3 
& 56.0 & 81.3 & 45.5 
& 40.6 & 62.4 & 25.3 \\
PA = 1
& 76.3 & 90.2 & 68.8 
& 72.2 & \textbf{89.8} & 64.8 
& 49.6 & \textbf{77.4} & 38.4 \\
\bottomrule
\end{tabular}}
\end{table}

\textbf{Balanced compensation preserves response quality.} Table~\ref{tab:ablation_pos_reward_final} examines positive advantage assignment. Removing positive compensation (PA = 0) lowers both faithfulness and informativeness across all three datasets. Fixed positive advantages (PA = 0.3 or 1) also yield lower faithfulness and Q-Score than BALTO, despite higher informativeness on RAGTruth and FinLLM-Eval. These fixed-PA settings isolate RLFH's fixed-reward design within our GRPO framework~\citep{wen2025policy}. The balanced default achieves the highest faithfulness and Q-Score among these variants.

\begin{table}[htbp]
\centering
\begin{minipage}[t]{0.48\linewidth}
\vspace{0pt}
\centering
\caption{Design and label-noise ablations (\%) with Qwen3-8B on FinLLM-Eval.}
\label{tab:design_noise_ablation}
\footnotesize
\setlength{\tabcolsep}{3pt}
\begin{tabularx}{\linewidth}{@{}>{\raggedright\arraybackslash}Xccc@{}}
\toprule
Setting & Faith. & Info. & Q-Score \\
\midrule
BALTO & \textbf{92.5} & \textbf{57.9} & \textbf{53.6} \\
Full claim penalty & 82.4 & 46.8 & 38.6 \\
Positive compensation $\times 0.5$ & 87.5 & 43.4 & 38.0 \\
Positive compensation $\times 2.0$ & 90.2 & 54.1 & 48.8 \\
Miss 20\% hallucinated tokens & 87.5 & 55.7 & 48.7 \\
Mislabel 20\% faithful tokens & 76.5 & 43.4 & 33.2 \\
\bottomrule
\end{tabularx}
\end{minipage}
\hfill
\begin{minipage}[t]{0.48\linewidth}
\vspace{0pt}
\centering
\caption{Human validation of the evaluator (\%). Token localization is measured on correctly identified hallucinated claims.}
\label{tab:evaluator_reliability}
\footnotesize
\setlength{\tabcolsep}{3pt}
\begin{tabularx}{\linewidth}{@{}>{\raggedright\arraybackslash}Xcc@{}}
\toprule
Level & Precision & Recall \\
\midrule
Claim extraction & 98.4 & 88.8 \\
Hallucination judgment & 81.3 & 92.4 \\
Error token localization & 94.6 & 89.8 \\
\bottomrule
\end{tabularx}
\end{minipage}
\end{table}

\textbf{The balanced ratio gives the best trade-off.} Table~\ref{tab:design_noise_ablation} shows that scaling the positive compensation $N^-/N^+$ by 0.5 or 2.0 also underperforms the balanced default (Q-Score: 38.0 and 48.8, respectively). Under-compensation reduces informativeness from 57.9 to 43.4, while over-compensation lowers faithfulness from 92.5 to 90.2.

\textbf{Penalizing entire claims reduces response quality.} Table~\ref{tab:design_noise_ablation} examines error localization and compensation strength with Qwen3-8B on FinLLM-Eval. Extending negative advantages from minimal erroneous spans to entire hallucinated claims reduces Q-Score from 53.6 to 38.6 and degrades both faithfulness and informativeness, supporting precise token localization.

\textbf{Mislabeling faithful tokens causes larger losses.} For the noise experiments in Table~\ref{tab:design_noise_ablation}, we randomly drop 20\% of hallucinated-token labels so they escape the negative penalty, or relabel 20\% of faithful tokens as hallucinated, assigning them negative advantages and removing them from the positive compensation set. Q-Score falls to 48.7 and 33.2, respectively, corresponding to relative decreases of 9\% and 38\%. Under these perturbations, faithful-token mislabeling is more damaging than missed hallucinations, highlighting the importance of accurate negative labels.

\subsection{Further Analysis}
\label{sec:analysis}

\textbf{Evaluator Reliability.} We conducted a 90-sample independent human annotation study (30 per dataset: ConFiQA, RAGTruth, and FinLLM-Eval) to evaluate three levels of the detection pipeline (Table~\ref{tab:evaluator_reliability}). The evaluator achieves high claim extraction quality and strong hallucination detection, with hallucination judgment recall reaching 92.4\%. Claim extraction and hallucination judgment use the same evaluator, Qwen3.5-35B-A3B, across GRPO$_B$, GRPO$_D$, FSPO, and BALTO, controlling the verifier model when comparing credit assignment mechanisms. For token localization, the evaluator uses its claim-verification reasoning to identify specific erroneous tokens, achieving 94.6\% precision and 89.8\% recall. These results support the reliability of BALTO's localized token-level feedback.

\begin{figure}[htbp]
  \centering
  \begin{minipage}[t]{0.48\textwidth}
    \centering
    \includegraphics[width=0.9\linewidth]{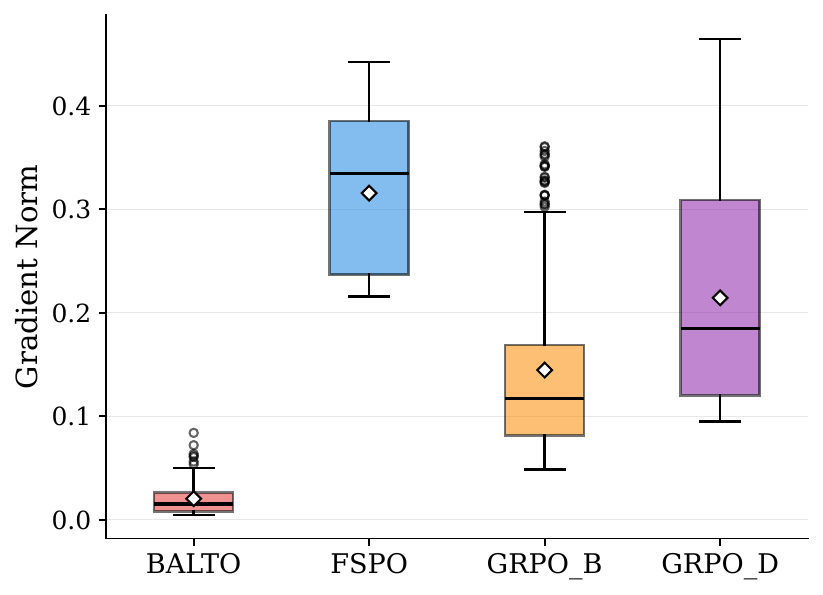}
    % \captionsetup{font=scriptsize}
    \caption{Gradient norm distribution on RAGTruth.}
    \label{fig:grad_norm_boxplot}
  \end{minipage}
  \hfill
  \begin{minipage}[t]{0.48\textwidth}
    \centering
    \includegraphics[width=0.9\linewidth]{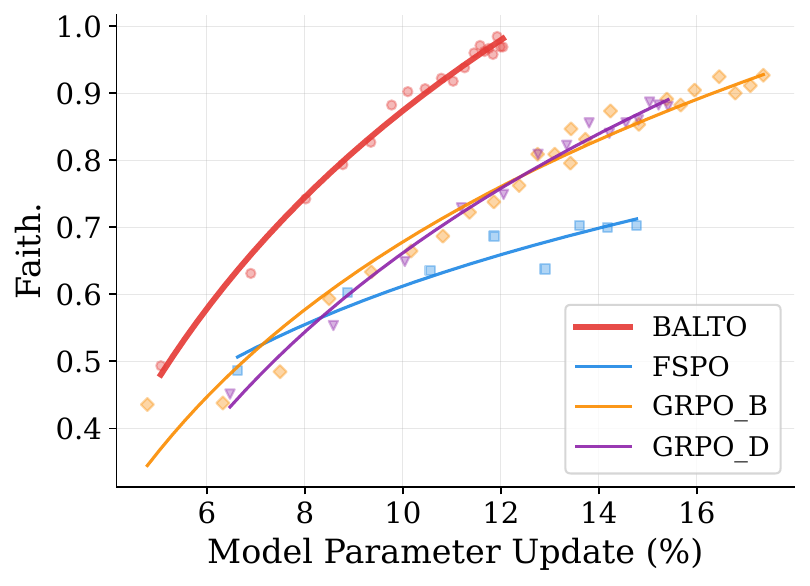} % Adjust image size to fit the available space
    % \captionsetup{font=scriptsize}
    \caption{Faithfulness vs. model parameter update on RAGTruth. Curves are fitted using logarithmic functions \(y = a \cdot \ln(x+1) + b\).}
    \label{fig:acc_vs_pc_val_single_RAGTruth}
  \end{minipage}
  
\end{figure}

\textbf{Training Dynamics.} Figure~\ref{fig:grad_norm_boxplot} shows the distribution of policy gradient norms across training steps on the RAGTruth dataset. Compared with response-level RL methods, BALTO exhibits gradient norms that are an order of magnitude smaller (mean 0.020 vs.\ 0.145--0.214) and more tightly concentrated. This behavior reflects the sparse and balanced token-level advantage structure in BALTO: only hallucinated and corresponding faithful tokens receive non-zero advantages, while all other neutral tokens are ignored. By restricting gradient updates to the most relevant tokens, BALTO reduces the influence of noisy signals from neutral tokens, resulting in lower variance and more stable optimization trajectories, as theoretically predicted in Theorem \ref{thm:stability}.

Figure~\ref{fig:acc_vs_pc_val_single_RAGTruth} plots faithfulness against the fraction of model parameters whose stored values differ from initialization \citep{mukherjee2025reinforcement}. All parameters remain trainable during RL. BALTO achieves 0.98 faithfulness with only approximately 12\% of parameters updated, whereas GRPO$_B$ requires around 17\% of parameters updated to reach 0.93 faithfulness. This indicates that token-level signals guide the optimizer to focus on a smaller and more relevant subset of parameters, specifically those that are directly responsible for generating hallucinated content, while maintaining the model's broader capabilities acquired during pre-training. The practical implication is that BALTO can reduce hallucinations more effectively while minimizing the risk of degrading general language performance.

\begin{figure}[htbp]
    \centering
    \includegraphics[width=\linewidth]{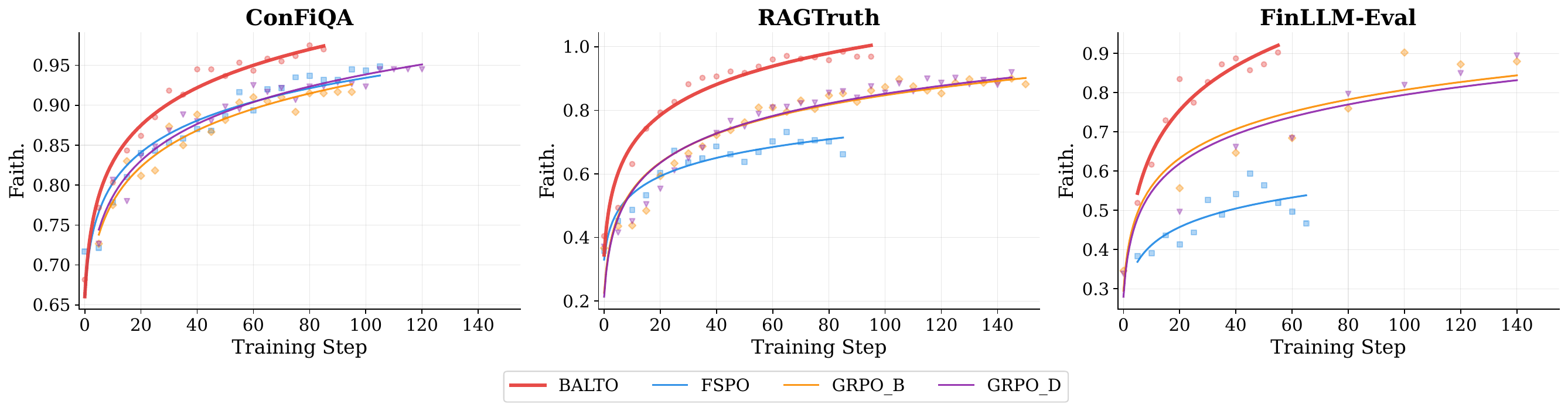}
    \caption{Faithfulness improvement during training across different RL methods on three datasets. Curves are fitted using logarithmic functions \(y = a \cdot \ln(x+1) + b\).}
    \label{fig:training_efficiency}
\end{figure}

\textbf{Efficiency.}
Figure~\ref{fig:training_efficiency} shows the improvement of faithfulness across training for different RL methods. BALTO consistently achieves faster early-stage gains and reaches higher final faithfulness with fewer training steps, reflecting the more informative token-level learning signals provided by balanced credit assignment. These observations align with Theorem \ref{thm:efficiency}, which demonstrates that BALTO maintains strong, bounded gradients across all stages of training, ensuring efficient exploration and stable convergence. In contrast, response-level methods such as GRPO$_B$ and GRPO$_D$ converge more slowly and are less effective on datasets with severe hallucinations, highlighting the advantage of fine-grained, balanced credit assignment for both performance and optimization efficiency.

\section{Conclusion}
In this work, we introduced BALTO, a balanced token-level policy optimization framework for fine-grained hallucination mitigation. Motivated by the localized nature of faithfulness hallucinations, BALTO addresses the credit assignment mismatch of response-level rewards by converting factuality judgments into token-level signals. By balancing positive and negative token-level credit assignment within each response, BALTO encourages probability mass to move from hallucinated content toward faithful content with balanced policy optimization.

Our theoretical analysis shows that balanced token-level credit assignment reduces unnecessary updates, lowers gradient variance, and provides bounded, stable optimization signals. Empirically, BALTO consistently improves the faithfulness--informativeness trade-off across ConFiQA, RAGTruth, and FinLLM-Eval with Qwen3-8B and Qwen3-4B, achieving stronger faithfulness and Q-Score than response-level and token-level RL baselines. These results highlight balanced fine-grained credit assignment as an effective direction for building more faithful and informative LLMs.

\begin{ack}

The SJTU team is supported by National Natural Science Foundation of China (62322603, 62502310) and Shanghai Municipal Science and Technology Major Project (2025SHZDZX025D08).

\end{ack}

\bibliographystyle{plainnat}
\bibliography{references}

%%%%%%%%%%%%%%%%%%%%%%%%%%%%%%%%%%%%%%%%%%%%%%%%%%%%%%%%%%%%

\newpage

\appendix

\section{Related Work}
\label{sec:related}
\subsection{Hallucination}

In LLMs, hallucinations refer to syntactically fluent outputs that contain factual errors or contradict source materials \citep{ji2023survey, alansari2026large}. Prior mitigation strategies are primarily data-centric or model-centric. Data-centric approaches optimize curation quality \citep{zhou2023lima}, employ specialized prompting \citep{jiang2023ai, sahoo2024systematic}, or utilize RAG~\citep{ding2024retrieve, bechard2024reducing} and reasoning techniques like Chain-of-Thought or self-consistency\citep{wang2022self, zhao2024enhancing}. Conversely, model-centric methods refine internal mechanisms through faithful fine-tuning \citep{hu2024mitigating, li2025know} and contrastive decoding \citep{chuang2024decoding}. More recently, RL has emerged as a significant direction \citep{wen2025policy, li2026reasoning}, which we detail below.

\subsection{RL for Faithfulness Alignment}

While RL is critical for model alignment, prior approaches often employ coarse-grained rewards \citep{chen2503research, song2025r1}, which may suffer from sparsity and imprecise guidance \citep{lightman2023let}. This has driven a shift towards fine-grained faithfulness alignment, where models are evaluated using claim-level correctness \citep{chen2025learning, dou2026baichuan} or atomic facts \citep{ren2025knowrl, yin2026mitigating}. However, these methods typically lack token-level precision or are constrained to specific domains. Furthermore, existing approaches that claim to utilize token-level rewards \citep{sunmechanistic, li2026reasoning, wen2025policy} either still depend on response-level signals or lack an effective credit balancing mechanism.
%To address these limitations, our work introduces a genuine token-level reward mechanism specifically designed for hallucination mitigation.

\subsection{Fine-Grained Rewards and Process Supervision}

To address the sparse reward issue inherent in traditional outcome-based models \citep{lightman2023let, zhang2026hcgrec}, recent literature has shifted towards fine-grained process supervision, primarily integrating intermediate signals into PPO or GRPO frameworks. Within PPO, approaches assign credit at the step-level to reinforce error localization \citep{cheng2025stop}, or at the token-level utilizing mechanisms such as attention distributions, edit distances, or discriminative optimization \citep{chan2024dense, yoon2024tlcr, chen2025discriminative}. Parallelly, GRPO-based methods explore diverse advantage estimation techniques, ranging from stepwise dynamic scaling and generative credit assignment \citep{zhang2025reward, xie2025capo} to token-level advantage computation via decaying accumulation, entropy-based shaping, or statistical measures \citep{he2025good, tan2025gtpo, sun2025ktae}. Despite these significant algorithmic advancements, genuinely integrating token-level rewards specifically for hallucination mitigation remains underexplored \citep{wen2025policy}, a critical gap that our work directly addresses.

\section{Limitations}
\label{sec:limitations}
Despite its effectiveness, BALTO has several limitations. First, our method relies on an LLM-based faithfulness evaluator to extract and verify factual claims, and thus its performance is bounded by the accuracy and consistency of this evaluator. Errors in claim extraction, evidence matching, or token-level projection may introduce noisy advantage assignments, especially for implicit claims, long-range reasoning dependencies, or ambiguous evidence. Second, BALTO is primarily designed for retrieval-grounded faithfulness alignment, where reference documents are available for verification. Its applicability to open-domain generation without explicit evidence, or to tasks involving subjective, creative, or preference-based outputs, remains less explored. 

\section{Broader Impact}
A key positive impact of BALTO is its ability to provide more precise supervision for hallucination mitigation. This may improve the usefulness of retrieval-augmented systems, financial assistants, medical question-answering systems, and other applications where factual precision is critical. In addition, BALTO's variance-reduction properties and efficient optimization behavior may reduce training instability and computational waste during RL-based alignment.

At the same time, our method also introduces risks. BALTO relies on an external faithfulness evaluator to identify hallucinated and faithful tokens. If the evaluator itself contains biases, factual inaccuracies, or systematic verification errors, these imperfections may propagate into the optimization process and reinforce incorrect supervision signals.

\section{Complete Mathematical Proofs for Theoretical Analysis}
\label{app:proofs}

In this appendix, we provide the complete mathematical derivations for the theorems presented in Section \ref{sec:theory}. We first establish the rigorous foundations of the policy gradient estimators used in our analysis.

\subsection{Preliminaries and Definition of Policy Gradient Estimators}

Let $y = (y_1, y_2, \dots, y_T)$ be a response of length $T$ generated by a language model policy $\pi_\theta(\cdot | x)$ given prompt $x$. The trajectory generation probability is given by $\pi_\theta(y \mid x) = \prod_{t=1}^T \pi_\theta(y_t \mid x, y_{<t})$.

We define the per-token score function (the gradient of the log-probability) as:
\begin{equation}
    g_t = \nabla_\theta \log \pi_\theta(y_t \mid x, y_{<t})
\end{equation}
According to the log-derivative trick, the expectation of the score function over the action space is strictly zero: $\mathbb{E}_{y_t\sim\pi_\theta(\cdot\mid x,y_{<t})}[g_t\mid x,y_{<t}]=0$.

The response-level score function, representing the gradient of the entire response, is the sum of the per-token score functions:
\begin{equation}
    G(y) = \nabla_\theta \log \pi_\theta(y \mid x) = \sum_{t=1}^T \nabla_\theta \log \pi_\theta(y_t \mid x, y_{<t}) = \sum_{t=1}^T g_t
\end{equation}

\textbf{Bernoulli Reward Setting.} To evaluate the optimization dynamics, we model the generation process as yielding two distinct types of trajectories:
\begin{itemize}
    \item \textbf{Faithful response ($y^+$):} Generated with probability $p$, yielding a reward $r=1$. It contains no hallucinations. We denote its response-level score function as $G(y^+) = \sum_{t=1}^{T^+} g_t^+$.
    \item \textbf{Hallucinated response ($y^-$):} Generated with probability $1-p$, yielding a reward $r=0$. It contains at least one hallucination. We denote its response-level score function as $G(y^-) = \sum_{t=1}^{T^-} g_t^-$.
\end{itemize}

Following our methodology, the token indices of any response are partitioned into three disjoint sets: the hallucinated factual token set $\mathcal{H}$ (size $N^-$), the faithful factual token set $\mathcal{F}$ (size $N^+$), and the neutral token set $\mathcal{N}$. For $y^+$, $N^- = 0$. For $y^-$, $N^- \ge 1$. Expectations $\mathbb{E}_{y^+}$ and $\mathbb{E}_{y^-}$ below are conditional on $r=1$ and $r=0$, respectively.

\textbf{Derivation of the Response-Level GRPO Estimator.}
The standard RL objective maximizes the expected reward: $J_{\mathrm{resp}}(\theta) = \mathbb{E}_{y \sim \pi_\theta}[R(y)]$.
Using the score identity and the population-normalized advantage, the reward gradient satisfies:
\begin{equation}
    \nabla_\theta J_{\mathrm{resp}}(\theta)
    =\mathbb{E}_{y\sim\pi_\theta}[(r-p)G(y)]
    =\sigma\,\mathbb{E}_{y\sim\pi_\theta}[A G(y)]
\end{equation}
where $A = \frac{r - \mu}{\sigma}$ is the normalized global advantage. Given $\mu = p$ and $\sigma = \sqrt{p(1-p)}$ for $0<p<1$, the respective advantages for $y^+$ and $y^-$ are:
\begin{equation}
    \label{eq:app_grpo_adv}
    A^+ = \sqrt{\frac{1-p}{p}}, \quad A^- = -\sqrt{\frac{p}{1-p}}
\end{equation}
We denote the raw normalized gradient signal for a single response by $\hat{\nabla}J_{\mathrm{resp}}$. Its expectation is $\nabla_\theta J_{\mathrm{resp}}/\sigma$, not the unnormalized reward gradient:
\begin{equation}
    \hat{\nabla} J_{\mathrm{resp}} = A \sum_{t=1}^T g_t
\end{equation}

\textbf{Derivation of the BALTO Estimator.}
BALTO assigns advantages at the token level. For the raw, unclipped surrogate with advantages held fixed, the expected gradient signal is:
\begin{equation}
    \mathbb{E}[\hat{\nabla}J_{\mathrm{BALTO}}] = \mathbb{E}_{y \sim \pi_\theta} \left[ \sum_{t=1}^T A_t^{\mathrm{bal}} \nabla_\theta \log \pi_\theta(y_t \mid x, y_{<t}) \right] = \mathbb{E}_{y \sim \pi_\theta} \left[ \sum_{t=1}^T A_t^{\mathrm{bal}} g_t \right]
\end{equation}
where the token-level advantages $A_t^{\mathrm{bal}}$ strictly follow the balanced assignment in Eq.~\ref{eq:balanced_advantage}. The corresponding raw gradient signal for a single response is:
\begin{equation}
    \hat{\nabla} J_{\mathrm{BALTO}} = \sum_{t=1}^T A_t^{\mathrm{bal}} g_t
\end{equation}

Here, $\hat{\nabla}J_m$ denotes a raw gradient signal rather than the gradient of the length-normalized, clipped objective in Section~\ref{sec:response_level_rl} or Eq.~\ref{eq:balto_objective}. The analysis is on-policy. Token weights and normalization statistics are held fixed during each surrogate update.

% ----------------------------------------------------------------------
\subsection{Proof of Theorem 1: Variance Reduction with Correlated Token Gradients}\label{proof1}

We analyze the raw gradient sums defined above, before clipping and length normalization, using the correlated-gradient assumptions (A1)--(A2) stated in Section~\ref{subsec:stability_variance}. Let $\mathcal{L}$ fix the response reward and token-label configuration for a response of length $T$. Advantage weights are deterministic within each such configuration, while token gradients may remain correlated.

\textbf{Conditional covariance decomposition.}
Write $\xi_t=g_t-\mathbb{E}[g_t\mid\mathcal{L}]$. Assumption (A1) states that
\begin{equation}
    \mathbb{E}[\xi_i\xi_j^\top\mid\mathcal{L}]
    =\rho_{ij}(\mathcal{L})\Sigma,
    \qquad \rho_{ii}=1,
\end{equation}
where $\Sigma$ is a common nonzero positive-semidefinite covariance matrix and $P=(\rho_{ij})$ is a correlation matrix. In particular, $|\rho_{ij}|\leq1$. No independence, isotropy, or uniform pairwise correlation is required. Centering is conditional: the unconditional score identity $\mathbb{E}[g_t]=0$ does not imply $\mathbb{E}[g_t\mid\mathcal{L}]=0$.

For arbitrary fixed weights $w_t$, the conditional covariance decomposes into diagonal and cross-correlation terms:
\begin{align}
    \operatorname{Cov}\left(\sum_{t=1}^T w_tg_t\mid\mathcal{L}\right)
    &=S(w,\mathcal{L})\Sigma, \\
    S(w,\mathcal{L})
    &=\underbrace{\sum_t w_t^2}_{D(w)}
      +\underbrace{\sum_{i\neq j}w_iw_j\rho_{ij}}_{C(w,\mathcal{L})}.
\end{align}
Taking the trace yields the scalar conditional variance
\begin{align}
    v_m(\mathcal{L})
    &=\operatorname{tr}\operatorname{Cov}(\hat{\nabla}J_m\mid\mathcal{L}) \nonumber\\
    &=\mathbb{E}\left[\left\lVert\hat{\nabla}J_m-\mathbb{E}[\hat{\nabla}J_m\mid\mathcal{L}]\right\rVert_2^2\mid\mathcal{L}\right]
      =S_m\sigma_g^2.
\end{align}
where $\sigma_g^2=\operatorname{tr}(\Sigma)>0$ and $S_m=S(w^{(m)},\mathcal{L})$ for the weights of method $m$. We compare $V_m=\mathbb{E}_{\mathcal{L}}[v_m(\mathcal{L})]$ for $m\in\{\mathrm{resp},\mathrm{BALTO}\}$. This conditioning keeps the reward-dependent weights inside their appropriate covariance calculation.

\textbf{BALTO upper bound.}
BALTO assigns $w_t=-1$ to hallucinated tokens, $w_t=N^-/N^+$ to faithful factual tokens when $N^-,N^+>0$, and zero to neutral tokens. For notational convenience, write $h=N^-$ and $f=N^+$. The diagonal term is
\begin{equation}
    D_{\mathrm{BALTO}}=h+f\left(\frac{h}{f}\right)^2
    =h+\frac{h^2}{f}.
\end{equation}
Using $|\rho_{ij}|\leq1$, we bound the cross term by grouping ordered pairs into hallucinated--hallucinated, faithful--faithful, and cross-set pairs:
\begin{align}
    |C_{\mathrm{BALTO}}|
    &\leq \underbrace{h(h-1)}_{\mathcal{H}\text{--}\mathcal{H}}
       +\underbrace{f(f-1)\left(\frac{h}{f}\right)^2}_{\mathcal{F}\text{--}\mathcal{F}}
       +\underbrace{2hf\frac{h}{f}}_{\mathcal{H}\text{--}\mathcal{F}} \nonumber\\
    &=4h^2-h-\frac{h^2}{f}.
\end{align}
Adding the diagonal term gives
\begin{equation}
    0\leq S_{\mathrm{BALTO}}
    =D_{\mathrm{BALTO}}+C_{\mathrm{BALTO}}
    \leq4h^2.
\end{equation}
Equivalently, this follows from $S_{\mathrm{BALTO}}\leq(\sum_t|w_t|)^2=(2h)^2$. The bound holds regardless of the signs or uniformity of individual correlations. When $h=0$, all weights vanish. When $h>0,f=0$, Eq.~\ref{eq:balanced_advantage} assigns only negative credit, yielding $S_{\mathrm{BALTO}}\leq h^2\leq4h^2$. Hence
\begin{equation}
    V_{\mathrm{BALTO}}\leq4\mathbb{E}[h^2]\sigma_g^2
    =4(1-p)\mathbb{E}[h^2\mid r=0]\sigma_g^2.
\end{equation}

\textbf{GRPO lower bound.}
GRPO applies the uniform weight $A$ to all $T$ tokens. By (A2),
\begin{equation}
    S_{\mathrm{resp}}=A^2\left[T+\sum_{i\neq j}\rho_{ij}\right]
    =A^2[T+T(T-1)\bar\rho]\geq A^2T.
\end{equation}
For $T=1$, the cross term is absent and the same lower bound holds. Under the Bernoulli reward model,
\begin{equation}
    \mathbb{E}[A^2]
    =p\frac{1-p}{p}+(1-p)\frac{p}{1-p}=1.
\end{equation}
Taking expectations therefore gives
\begin{equation}
    \label{eq:var_grpo}
    V_{\mathrm{resp}}\geq T\sigma_g^2.
\end{equation}

\textbf{Comparison and scope.}
Combining the two bounds,
\begin{equation}
    \frac{V_{\mathrm{BALTO}}}{V_{\mathrm{resp}}}
    \leq\frac{4\mathbb{E}[(N^-)^2]}{T}.
\end{equation}
Thus, $4\mathbb{E}[(N^-)^2]<T$ implies $V_{\mathrm{BALTO}}<V_{\mathrm{resp}}$. For fixed hallucination counts, this becomes $N^-<\sqrt{T}/2$. Compared with the original independence-based analysis, the correlated case has the more conservative quadratic dependence on $N^-$. Sparsity expressed only as a small ratio $N^-/T$ does not by itself guarantee this condition.

The comparison concerns fluctuations within reward/label configurations. By the law of total variance,
\begin{equation}
    \operatorname{tr}\operatorname{Cov}(\hat{\nabla}J_m)
    =V_m+\operatorname{tr}\operatorname{Cov}_{\mathcal{L}}
      \left(\mathbb{E}[\hat{\nabla}J_m\mid\mathcal{L}]\right).
\end{equation}
Without an additional bound on the second term, the theorem does not order the total scalar gradient variance. For the length-normalized, unclipped objectives, the corresponding scalar conditional variances are $S_{\mathrm{resp}}\sigma_g^2/T^2$ and $S_{\mathrm{BALTO}}\sigma_g^2/Z^2$, where $Z=\max(1,N^-+N^+)$. These normalization factors and clipping must be considered separately when relating the raw-sum comparison to training. The result provides formal motivation under (A1)--(A2), rather than a tight guarantee for the implemented optimizer. \hfill $\blacksquare$

% ----------------------------------------------------------------------
\subsection{Proof of Theorem 2: High Optimization Efficiency in Asymptotic Regimes}\label{proof2}

Optimization efficiency relies on the expected magnitude of the policy gradient providing a strong, informative signal. We analyze this expectation at the two extremes of alignment training: $p \to 0$ (cold-start) and $p \to 1$ (late-stage convergence).

\textbf{Part I: Gradient Starvation in the Cold-Start Regime ($p \to 0$)}

For Response-Level GRPO, we expand the expected gradient over the Bernoulli outcomes using the predefined response-level score functions $G(y^+)$ and $G(y^-)$:
\begin{align}
    \mathbb{E}[\hat{\nabla} J_{\mathrm{resp}}] &= p \cdot \mathbb{E}_{y^+}\left[ A^+ G(y^+) \right] + (1-p) \cdot \mathbb{E}_{y^-}\left[ A^- G(y^-) \right] \nonumber \\
    &= p \sqrt{\frac{1-p}{p}} \mathbb{E}_{y^+}[G(y^+)] - (1-p) \sqrt{\frac{p}{1-p}} \mathbb{E}_{y^-}[G(y^-)] \nonumber \\
    &= \sqrt{p(1-p)} \Big( \mathbb{E}_{y^+}[G(y^+)] - \mathbb{E}_{y^-}[G(y^-)] \Big)
\end{align}
The overall magnitude of the learning signal is constrained by the global normalization coefficient $\sqrt{p(1-p)}$. Taking the limit as the faithful rate $p \to 0$:
\begin{equation}
    \lim_{p \to 0} \lVert\mathbb{E}[\hat{\nabla} J_{\mathrm{resp}}]\rVert_2 \propto \lim_{p \to 0} \sqrt{p(1-p)} = 0
\end{equation}
This mathematically demonstrates \textit{gradient starvation}: when a task is exceedingly difficult and global rewards are uniformly $0$, standard GRPO normalizes the advantages to zero, providing no effective optimization signal to correct errors.

Conversely, for BALTO, faithful responses $y^+$ yield zero advantage. The gradient expectation strictly depends on the hallucinated responses $y^-$:
\begin{equation}
    \mathbb{E}[\hat{\nabla} J_{\mathrm{BALTO}}] = p \cdot 0 + (1-p) \cdot \mathbb{E}_{y^-} \left[ \sum_{t=1}^T A_t^{\mathrm{bal}} g_t \right]
\end{equation}
The magnitude of BALTO's expected gradient is governed by the pre-factor $(1-p)$. Taking the limit as $p \to 0$:
\begin{equation}
    \lim_{p \to 0} \lVert\mathbb{E}[\hat{\nabla} J_{\mathrm{BALTO}}]\rVert_2 \propto \lim_{p \to 0} (1-p) = 1
\end{equation}
BALTO achieves its theoretical maximum expected gradient magnitude precisely when the model's performance is at its lowest, guaranteeing highly efficient exploration and parameter updates during the cold-start phase.

\vspace{2mm}
\textbf{Part II: Divergent Penalties in the Convergence Regime ($p \to 1$)}

We now analyze the penalty assigned to a rare, hallucinated trajectory $y^-$ as the model approaches near-perfect alignment ($p \to 1$).
For standard GRPO, the penalty is determined by $A^-$:
\begin{equation}
    \lim_{p \to 1} A^- = \lim_{p \to 1} -\sqrt{\frac{p}{1-p}} = -\infty
\end{equation}
This divergence reveals a critical instability: a singular, rare mistake triggers an infinitely large negative update applied globally across the entire response. This disproportionate penalty destroys late-stage optimization efficiency by destabilizing previously well-aligned parameters.

In stark contrast, the advantages in BALTO are deterministically assigned based on local token sets according to Eq.~\ref{eq:balanced_advantage}:
\begin{equation}
    A_t^{\mathrm{bal}} \in \left\{ -1, \frac{N^-}{N^+}, 0 \right\}
\end{equation}
Because hallucinations are a sparse minority of the response ($N^- \ll N^+$), the compensation advantage $N^- / N^+$ is a small positive fraction. Thus, all local advantages strictly remain bounded within $[-1, 1]$, entirely independent of the global faithful rate $p$.

Furthermore, analyzing the expected gradient magnitude for BALTO as $p \to 1$:
\begin{equation}
    \lim_{p \to 1} \lVert\mathbb{E}[\hat{\nabla} J_{\mathrm{BALTO}}]\rVert_2 \propto \lim_{p \to 1} (1-p) = 0
\end{equation}
This proves that BALTO ensures efficient and graceful final convergence: as the model ceases to hallucinate, the gradient expectation naturally decays to zero without ever triggering extreme penalties or gradient explosion. \hfill $\blacksquare$

\section{Implementation Details}\label{app:implementation-details}

\subsection{Baselines.}
\paragraph{Surprised Fine-Tuning (SFT).} When constructing the SFT dataset, we first generated responses using the base model. We then employed Qwen3.5-35B-A3B as a reward model to detect hallucinations within these responses, with the prompt in Appendix \ref{prompt:hallucination-detection}. For responses identified as containing hallucinations, we further used Qwen3.5-35B-A3B guided by the detection results to correct the original outputs, aiming to preserve the original content structure as much as possible. The prompt used for this correction process is provided in Appendix \ref{prompt:response-correction}. Both the original correct responses and the corrected responses were subsequently used as the training data for supervised fine-tuning.

\paragraph{Direct Preference Optimization (DPO).} We constructed preference pairs using the SFT dataset: the negative samples were generated by the base model during the SFT data construction process, while the positive samples were obtained from the corrected responses. These preference pairs were then used to train the model with the DPO objective.

\paragraph{Group Relative Policy Optimization (GRPO).}  In the training process, we employed Qwen3.5-35B-A3B as a reward model to detect hallucinations within responses, with the prompt in Appendix \ref{prompt:hallucination-detection}. For GRPO$_B$, it gives a 0/1 reward to indicate if the response has hallucination. For GRPO$_D$, we calculate the ratio of faithful claims as the reward.

\paragraph{FSPO.} Similar to GRPO$_D$, we calculate the ratio of faithful claims as the reward to calculate the advantage. The advantage of the faithful claims in responses with negative advantage will be inverted. Similarily, The advantage of the hallucinated claims in responses with positive advantage will be inverted. 

\paragraph{BALTO.} As the same setting, we use the reward model to detect hallucination tokens within responses, with the prompt in Appendix \ref{prompt:hallucination-detection}.

\subsection{Token Alignment}
\label{app:token_alignment}

The evaluator extracts claims, verifies them, and outputs erroneous text segments as \texttt{error\_spans} in a single prompt call. We align these segments with the original rollout tokens in two phases.

\textbf{Span location.} We locate each substring in the response through cascaded matching: exact substring search, Markdown-insensitive matching, whitespace-insensitive matching, and combined Markdown, whitespace, and punctuation-normalized matching. The matched spans are mapped to character positions in the original response.

\textbf{Character-to-token mapping.} We re-encode the response with the training tokenizer and identify tokens overlapping the character spans. If the re-encoded token IDs match the original rollout IDs, tokenizer offsets directly provide the token indices. Otherwise, we decode the original token sequence prefix by prefix to build a character-to-token mapping and project the spans onto the original rollout tokens. If the minimal erroneous part cannot be reliably isolated, we use the entire hallucinated claim as described in Section~\ref{sec:detection}.

\subsection{Evaluation}
\label{app:evaluation}

\paragraph{LLM-based Evaluation.}
We use Qwen3.5-35B-A3B as the evaluator to extract claims from model responses and determine whether each claim is supported by the reference evidence. 
The detailed prompts used for claim extraction and hallucination detection are provided in Appendix \ref{prompt:hallucination-detection}. 

\paragraph{Faithfulness.}
Faithfulness (Faith.) is calculated as the proportion of responses without hallucinated claims over the test set. 

\paragraph{Informativeness.}
Informativeness (Info.) measures whether a response preserves sufficient information compared to a reference response (generated by the base model). 
For each example, a response receives a score of 1 if it contains at least as many claims as the reference. Otherwise, it receives a score of 0. The final Info. score is averaged over the test set.

\paragraph{Q-Score.}
To jointly evaluate faithfulness and informativeness, we define:
Q-Score = Faith. $\times$ Info., which reflects the overall response quality by jointly considering correctness and completeness. The multiplicative form explicitly penalizes responses that achieve high performance on only one dimension.

\paragraph{ConFiQA Accuracy.}
For ConFiQA, the paper \citep{bi2025context} provides two metrics: $P_c$ is the frequency of responses matching the context-faithful answer (excluding negations or the original answer), and $P_o$ is the frequency of responses matching the original factual answer. We calculate the accuracy (Acc.) as:
$\text{Acc.} = (P_c - P_o + 100)/2$, where $P_c$ and $P_o$ are expressed as percentages. A higher Acc. indicates better performance.

\paragraph{Win Rate.}
We follow prior work~\citep{li2026knowledgelevel, chen2025learning} to compute the win rate (WR) using LLM-based pairwise comparison. 
We use DeepSeek-v3.2 \citep{liu2025deepseek} as the judge model with the prompt in Appendix \ref{prompt:response-comparison}. For each prompt, responses from two models are compared twice with reversed ordering to mitigate position bias. 
Let $\text{Win}(1)$ and $\text{Win}(2)$ denote whether the candidate model is preferred in each ordering. 
The instance-level win indicator is defined as:
\[
\text{Win} = \frac{\text{Win}(1) + \text{Win}(2)}{2}.
\]
The overall win rate over a test set of size $N$ is:
\[
\text{WR} = \frac{1}{N} \sum_{i=1}^{N} \text{Win}_i.
\]
A win rate greater than 0.5 indicates that the candidate model is preferred over the compared model.

\subsection{Experiments Compute Resources}
We use 2x8 Nvidia H20 to conduct the experiments. The training time of Qwen3-8B and Qwen3-4B in one epoch is around 30 minites.

\section{Additional Results}
\label{app:additional_experiments}

\subsection{Generalization to Llama-3.1-8B}
\label{app:llama_results}

We conduct additional experiments with Llama-3.1-8B~\citep{grattafiori2024llama} on RAGTruth and FinLLM-Eval to evaluate BALTO beyond the Qwen model family. Table~\ref{tab:llama_results} reports the results.

\begin{table}[htbp]
\centering
\caption{Llama-3.1-8B results (\%) on RAGTruth and FinLLM-Eval.}
\label{tab:llama_results}
\small
\begin{tabular}{lcccccc}
\toprule
\multirow{2}{*}{Method} & \multicolumn{3}{c}{RAGTruth} & \multicolumn{3}{c}{FinLLM-Eval} \\
\cmidrule(lr){2-4}\cmidrule(lr){5-7}
& Faith. & Info. & Q-Score & Faith. & Info. & Q-Score \\
\midrule
Base & 39.6 & \textit{[ref.]} & 39.6 & 41.9 & \textit{[ref.]} & 41.9 \\
GRPO$_B$ & 73.6 & 46.0 & 33.9 & 71.0 & 54.8 & 38.9 \\
GRPO$_D$ & 91.3 & 53.6 & 48.9 & 69.9 & 67.7 & 47.3 \\
FSPO & 82.7 & \textbf{77.1} & 63.8 & 68.8 & \textbf{68.8} & 47.3 \\
BALTO & \textbf{96.7} & 69.1 & \textbf{66.8} & \textbf{76.3} & 63.4 & \textbf{48.4} \\
\bottomrule
\end{tabular}
\end{table}

BALTO achieves the highest faithfulness and Q-Score on both datasets. On RAGTruth, it reaches 96.7\% faithfulness and 66.8\% Q-Score. On FinLLM-Eval, it achieves 76.3\% faithfulness and 48.4\% Q-Score. These results support the effectiveness of balanced token-level credit assignment beyond the Qwen model family.

\subsection{Results on MT-RAG}
\label{app:mtrag_results}

We evaluate BALTO on MT-RAG~\citep{katsis2025mtrag} using reference-answer-based lexical metrics: Precision, Recall, Token F1, and ROUGE-L. Table~\ref{tab:mtrag_results} shows that BALTO achieves the highest scores across all four metrics. Its ROUGE-L improves by 4.78 points over the base model, compared with 2.74 for GRPO$_D$ and 1.66 for FSPO. These results provide complementary evidence of improved response quality through better overlap with reference answers on an additional benchmark.

\begin{table}[htbp]
\centering
\caption{Results on MT-RAG (\%).}
\label{tab:mtrag_results}
\small
\begin{tabular}{lcccc}
\toprule
Method & Precision & Recall & Token F1 & ROUGE-L \\
\midrule
Base & 35.03 & 44.69 & 36.26 & 23.58 \\
GRPO$_D$ & 39.16 & 44.40 & 38.40 & 26.32 \\
FSPO & 37.20 & 45.53 & 38.00 & 25.24 \\
BALTO & \textbf{40.75} & \textbf{47.39} & \textbf{40.42} & \textbf{28.36} \\
\bottomrule
\end{tabular}
\end{table}

\subsection{Frontier Model Reference Results}
\label{app:frontier_results}

We evaluate DeepSeek-v3.2-thinking and Claude-Opus-4.8 on all three benchmarks using the same faithfulness evaluation protocol. Table~\ref{tab:frontier_results} provides reference results alongside Qwen3-8B. BALTO-trained Qwen3-8B achieves higher faithfulness than both frontier models on ConFiQA and RAGTruth, while the frontier models perform better on FinLLM-Eval. The ConFiQA results are consistent with BALTO's explicit training to stay grounded in the provided counterfactual context.

\begin{table}[htbp]
\centering
\caption{Faithfulness reference results (\%) under the same evaluation protocol.}
\label{tab:frontier_results}
\small
\begin{tabular}{lccc}
\toprule
Model & ConFiQA & RAGTruth & FinLLM-Eval \\
\midrule
Qwen3-8B (Base) & 68.2 & 40.4 & 34.6 \\
Qwen3-8B (BALTO) & \textbf{97.5} & \textbf{98.4} & 92.5 \\
DeepSeek-v3.2-thinking & 84.9 & 88.8 & 94.4 \\
Claude-Opus-4.8 & 92.9 & 96.3 & \textbf{96.8} \\
\bottomrule
\end{tabular}
\end{table}

\subsection{Training Compute Overhead}
\label{app:compute_overhead}

Hallucination detection runs on the fly during RL training: each batch generates $K=8$ responses per prompt, verifies their claims with Qwen3.5-35B-A3B, projects the judgments to token labels, and computes balanced advantages for the policy update. GRPO$_D$ and BALTO use the same claim extraction and verification call. BALTO additionally requests the specific erroneous spans through the \texttt{error\_spans} field.

We measure the marginal detection cost with and without \texttt{error\_spans} using the same evaluator on five random samples from each of the three datasets. Token localization increases detection time by 3.0\% on average, or 0.12 seconds per sample. With reward computation accounting for approximately 32\% of total step time in the measured runs, this corresponds to an estimated 1\% increase in total training step time. The small marginal cost reflects that the evaluator already reasons about the errors during claim verification. Figure~\ref{fig:training_efficiency} further shows that BALTO reaches higher faithfulness with fewer training steps.

\section{Discussions}
\label{app:discussions}

\subsection{Possible Extension to PPO}
\label{app:ppo_extension}

BALTO targets the credit mismatch caused by broadcasting a response-level advantage to all tokens. Its localized corrective design can also be explored in PPO by treating the balanced credits in Eq.~\ref{eq:balanced_advantage} as token-level factuality rewards, $r_t^{\mathrm{fact}}=A_t^{\mathrm{bal}}$: hallucinated tokens receive $-1$, faithful factual tokens receive $N^-/N^+$ when both sets are nonempty, and neutral tokens receive zero. The boundary cases follow the same assignment rule. Standard GAE can then compute advantages from these rewards and the value estimates, while the value model learns factuality-aware returns and KL penalties enter as additive reward terms. This preserves the localized reward construction. Bootstrapping, GAE, and KL terms need not preserve its zero-sum property in the resulting advantages. This extension provides a path to actor--critic training, while its interaction with value learning remains to be empirically evaluated.

\subsection{Extension without Pre-existing Reference Documents}
\label{app:reference_free_extension}

BALTO's token-level credit assignment can also be explored when reference documents are not provided in advance. FActScore~\citep{min2023factscore} and SAFE~\citep{wei2024long} provide claim-level factuality verification using retrieved evidence. A possible extension is to extract atomic claims, retrieve supporting evidence, and ask the verifier to return both factuality judgments and minimal erroneous text spans. The token alignment procedure in Appendix~\ref{app:token_alignment} then projects these spans to rollout tokens, after which the same balanced advantage rule can be applied. This integration would enable fully online RL training without pre-existing reference documents.

\section{Prompts}\label{app:prompts}

\subsection{Hallucination Detection}\label{prompt:hallucination-detection}
\begin{promptbox}{Hallucination Detection}
# Role
You are a strict auditor of response accuracy. Your core competency lies in meticulously cross-referencing statements within a response text against provided reference materials. Your objective is to verify whether every informational statement in the response text aligns with the reference materials, and to precisely pinpoint any specific erroneous segments.

# Task
I will provide a "Response Text" and a set of "Reference Materials." Your task is to verify, sentence by sentence, every informational statement within the Response Text -- specifically validating whether its assertions are consistent with the content found in the Reference Materials -- and to precisely identify any specific erroneous segments.

# Input Data
## [User Query & Reference Materials]
{{user_prompt}}

## [Response Text]
{{response}}

# Workflow
1. **Extraction:** Extract all informational sentences from the "Response Text." 
- Extracted sentences typically contain specific data points -- such as figures (dates, times, and various other numerical values), entities (people, places, venues, etc.), and logical relationships (e.g., whether a cause-and-effect link holds true).
2. **Localization:** Locate the corresponding paragraphs within the "Reference Materials."
3. **Comparison:** Compare the statements made in the response against the original wording found in the reference materials.
4. **Verdict:**
- **Correct:** The statement aligns perfectly with the content found in the reference materials. 
- **Incorrect:** A corresponding statement exists in the reference materials, but the response's phrasing does not match it; *or* the statement cannot be found within the specified reference materials (i.e., it is fabricated/hallucinated).
5. **Error Localization:** For every incorrect item, further pinpoint the specific "erroneous segment" (`error_spans`) within the "Response Text" itself, to facilitate subsequent token-level penalties. 
- For each incorrect item, you must identify the minimal erroneous segment within the [Response Text] and output it to the `error_spans` list. 
- Each entry in `error_spans` must be a contiguous substring taken directly from the original Response Text. 
- If a single error spans multiple non-contiguous segments, split them into separate elements and add each to the `error_spans` list. 
- Avoid marking entire sentences as erroneous segments; instead, identify only the minimal, specific error fragments. 
- If no errors are detected, the `error_spans` list must be empty. # Possible Errors All possible data accuracy error types:

- conflict: Factual contradiction (the model generates information that contradicts facts in the references, such as misattribution of dates, metrics, entities, etc.)

- fabrication: Fabrication (the model generates information that is not mentioned in the original text/context)

- no_error: No error

# Output Format Please analyze carefully. The output should be in the following format. Strictly adhere to the format and do not output any extra symbols:

### Information Accuracy Analysis

{
"details": [

{
"claim_text": "The extracted sentence containing the information",

"source_text": "The corresponding sentence in the references",

"analysis": "Analyze whether there is a match or whether the calculation result is correct",

"error_type": "One of the error types in Possible Errors: conflict/fabrication/no_error",

"judgment_result": "Correct/Incorrect",

"error_spans": ["Error Span 1", "Error Span 2", ...]

},

{ "claim_text": "The extracted sentence containing the information",

"source_text": "The corresponding sentence from the references",

"analysis": "Analysis of whether a match is found or whether the calculation result is correct",

"error_type": "One of the error types in Possible Errors: conflict/fabrication/no_error",

"judgment_result": "Correct/Incorrect",

"error_spans": ["Error Span 1", "Error Span 2", ...]

},...

]
}
### Summary of Response Accuracy The response contains X claims, of which Y are correct matches and Z are incorrect matches (Note that this summary must strictly adhere to the format; do not output any extra symbols). If illusions exist, summarize the illusion points.
\end{promptbox}

\subsection{Response Correction}\label{prompt:response-correction}
\begin{promptbox}{Response Correction}
# Role
You are a rigorous text proofreading and fact-checking expert. Your task is to verify and correct numerical or factual errors -- identified in the [Hallucination Detection Report] -- within the [Original Reasoning Process] and [Original Answer], based on the provided [Reference Materials].

# Rules
When making modifications, please strictly adhere to the following workflow and editing principles:

## Step 1: Verify the Accuracy of Hallucination Detection
Before making any changes, you must first cross-reference the [Reference Materials] to determine whether the errors flagged in the [Hallucination Detection Report] actually exist.
- If the [Hallucination Detection Report] is accurate (i.e., the original text is indeed incorrect), proceed to Step 2 to make the necessary corrections.
- If the [Hallucination Detection Report] is a false positive (i.e., the original text actually aligns with the reference materials), retain the original text and do not make that specific correction.

## Step 2: Apply the "Minimal Change" Principle
If you confirm that the original text contains an error, you must correct it in both the [Original Reasoning Process] and the [Original Answer]. When making corrections, strictly follow the priority order below (from highest to lowest):
1. **Precise Numerical Replacement (Primary Choice)**: Modify only the incorrect number (e.g., `17.96%` -> `12.08%`).
2. **Precise Entity Replacement (Secondary Choice)**: If changing only the number results in a semantic contradiction (e.g., a subject mismatch), make minor adjustments to the local subject entity name or qualifier (e.g., `Bond D` -> `Bond A`).
3. **Local Refinement (Tertiary Choice)**: If simply replacing the number or entity leads to grammatical errors, semantic incoherence, or logical conflicts within the context, you are permitted to make minor additions, deletions, or stylistic tweaks to the specific sentence -- or its immediate context -- to ensure it reads smoothly.
4. **Full Sentence Deletion (Last Resort)**: If the erroneous information is a complete fabrication (a "hallucination") that cannot be factually corrected through simple edits -- or if correcting it would cause the logic of the entire paragraph to collapse -- directly delete the entire sentence containing the error.

## Step 3: Formatting Requirements
- Preserve the original text structure, tone, and formatting (e.g., Markdown headings, lists, etc.). - The citation markers in the original text (such as `[12](@ref)`) must remain accurate after modification. If a sentence is deleted, please also delete the citation markers that belong to that sentence.

---

# Input Data

### [References]

${user_prompt}

### [Hallucination Detection Results]

${detect_result}

### [Original Reasoning Process]

${reasoning_content}

### [Original Answer]

${answer}

---

# Output Format
Please output your results in the following structure:

### 1. Verification and Modification Explanation

(Briefly explain your verification results of the [Hallucination Identification Report] and what level of modification strategy you adopted: replacement/fine-tuning/deletion.)

### 2. Modified Reasoning Process

(Output the corrected complete reasoning_content)

### 3. Modified Answer

(Output the corrected complete answer)
\end{promptbox}

\subsection{Response Comparison}\label{prompt:response-comparison}
\begin{promptbox}{Response Comparison}
You are a helpful assistant, that ranks models by the quality of their answers.

I want you to create a leaderboard of different of large-language models. To do so, I will give you the instructions (prompts) given to the models, and the responses of two models. Please rank the models based on which responses would be preferred by humans. All inputs and outputs should be python dictionaries.

Here is the prompt:
{
    "instruction": """{{instruction}}"""
}

Here are the outputs of the models:
[
    {
        "model": "model_1",
        "answer": """{{output_1}}"""
    },
    {
        "model": "model_2",
        "answer": """{{output_2}}"""
    }
]

Now please analyze the two models' answers and return the better model.

Output format:

### Analysis
xxx (your analysis)

### Better Model
model_1 / model_2 / tie (The model you think is better. If you think they are equally good, return "tie")
\end{promptbox}

%%%%%%%%%%%%%%%%%%%%%%%%%%%%%%%%%%%%%%%%%%%%%%%%%%%%%%%%%%%%

\end{document}